\documentclass[letterpaper, 10 pt, conference]{ieeeconf}  

\IEEEoverridecommandlockouts                              

\usepackage{times}
\usepackage{helvet}
\usepackage{courier}
\usepackage{algorithm}
\usepackage{algorithmic}
\usepackage{balance}
\usepackage{amsmath}
\usepackage{amssymb}
\usepackage{booktabs}
\usepackage{graphicx}
\usepackage{array}
\usepackage{subcaption} 

\newtheorem{theorem}{Theorem}
\newtheorem{lemma}[theorem]{Lemma}

\newenvironment{myproof}[1][\proofname]{\proof[#1]\mbox{}}{\endproof}

\title{\LARGE \bf
No-Regret Replanning under Uncertainty
}

\author{Wen Sun$^{1}$,  Niteesh Sood$^{2}$, Debadeepta Dey$^{3}$, Gireeja Ranade$^{3}$, \\ Siddharth Prakash$^{2}$, and Ashish Kapoor$^{3}$  
\thanks{$^{1}$The Robotics Institute, School of Computer Science, Carnegie Mellon University. The research work was done during an internship at Microsoft Research, Redmond.  {\tt\small wensun@cs.cmu.edu}. }%
\thanks{$^{2}$Microsoft Research, India}%
\thanks{$^{3}$Microsoft Research, Redmond}
}

\begin{document}

\maketitle
\thispagestyle{empty}
\pagestyle{empty}

\begin{abstract}
This paper explores the problem of path planning under uncertainty. Specifically, we consider online receding horizon based planners that need to operate in a latent environment where the latent information can be modeled via Gaussian Processes. Online path planning in latent environments is challenging since the robot needs to explore the environment to get a more accurate model of latent information for better planning later and also achieves the task as quick as possible. We propose UCB style algorithms that are popular in the bandit settings and show how those analyses can be adapted to the online robotic path planning problems. The proposed algorithm trades-off exploration and exploitation in near-optimal manner and has appealing no-regret properties. We demonstrate the efficacy of the framework on the application of aircraft flight path planning when the winds are partially observed.
\end{abstract}

\section{INTRODUCTION}
Finding an optimal path under unknown or partially observed environment is a challenging and an important task in robotics. In this paper, we consider an online replanning framework where in each round, the robot picks a direction to traverse and as it travels, it receives observations about unknown variables along the trajectory.
The robot then considers this newly acquired information to refine it knowledge about the environment, which in turn influences the action selection in the next round.  Finding an optimal strategy is challenging in such online replanning framework as the robot essentially faces a tradeoff between exploration and exploitation. In order to make inferences about the latent variables, the robot needs to pick actions that can drive itself around the space to gather information. Such exploration can be beneficial as more accurate knowledge of the environment promises more accurate estimation of the cost of trajectories. However, such actions come at a cost, especially if these information foraging actions make the robot deviate from its mission. Thus, it is important for the robot to make the right decisions about when/where to explore and when to exploit.

We specifically focus on receding-horizon replanning, where the robot is equipped with a pre-computed library of trajectories and planning entails picking a trajectory among the library at every round. Also, we consider the cases where the uncertainties in environment are unknown but can be approximately modeled by Gaussian Processes. There is a fairly large and important classes of natural phenomenon, including winds, oceanic currents and traffic volume that is spatially correlated, can be modeled with GPs. For example, consider an aerial robot that is attempting to minimize traversal time by exploiting the tail winds while avoiding the head winds. However, without complete information it is difficult for the robot to make the optimal decision. Consequently, we ask how should the aircraft traverse in the space in order to maximally utilize the winds while continuously sensing and updating its belief about the wind field.

This paper addresses such exploration-exploitation trade-off in the online replanning framework  by presenting a new and simple replanning strategy called \emph{Upper Confidence Bound} Replanning (UCB-Replanning). When trading between exploration and exploitation, UCB-Replanning uses the classic strategy of \emph{Optimism in the Face of Uncertainty}. Since Gaussian Processes provide a distribution over the latent variable of interest, UCB-Replanning can leverage the inferred uncertainty to build confidence intervals on the estimations of the cost of trajectories. During replanning, the robot then takes the estimation of the cost and the confidence interval of the estimation together into consideration to make a decision. We further analyze the performance of UCB-Replanning. Particularly, we show that UCB-Replanning is \emph{no-regret} in a sense that the UCB-Replanning in average is doing almost as well as picking the optimal trajectory assuming the full knowledge of unknown variables, along the states that the robot has visited. 

Finally we conduct a case study of aircraft navigating under wind uncertainty, where the wind speed is modeled by a Gaussian Process. We investigate the experimental performance of UCB-replanning under different types of wind map, including wind maps over the continental United States, constructed from real wind data provided by National Oceanic and Atmospheric Administration (NOAA). 

\section{RELATED WORK}
We discuss UCB-Replanning in relation to literature in three main areas: 1. Receding-horizon planning in robotics 2. Partially Observable Markov Decision Processes (POMDPs) and 3. Multi-Armed Bandit problems (MAB).

\noindent \textbf{Receding-horizon planning}: In receding-horizon control a library of pre-computed control command sequences are simulated forward from the current state of the robot using the dynamic motion model to come up with a set of dynamically feasible trajectories up to the planning horizon. This set of trajectories is then evaluated on the map of the world in the vicinity of the robot and amongst all the currently collision-free trajectories the one that makes most progress towards the goal is chosen for traversal \cite{green2007toward}. The selected trajectory is traversed for a portion of the time and the process of trajectory evaluation and selection is repeated again. Receding-horizon based planning has been widely used in aerial and ground robot navigation in cluttered environments \cite{urmson2008autonomous,dey2015vision,arora2015emergency} due to many attractive properties like finite runtime, adaptability to available computational budget and dynamic feasibility by construction. We use receding-horizon planning with pre-computed trajectory libraries as the  framework in this paper.

\noindent \textbf{Partially Observable Markov Decision Processes (POMDPs)}: POMDPs are used to model Markov Decision Processes (MDPs) where only part of the state of the world can be observed. Finding optimal policies of POMDP is NP-hard \cite{Papadimitriou:1987}. Approximate solutions like Point-Based Value Iteration \cite{pineau2003point,kurniawati2008sarsop}, Heuristic Search-Based Value Iteration \cite{smith2004heuristic} and Monte-Carlo planning \cite{silver2010monte} are popular goal-free reward oriented solvers. While goal-oriented methods like \cite{bonet2009sol} are more relevant to our problem scenario, they are hard to adapt to continuous observation spaces and computation and time budgets imposed by mobile robots. Belief Space Planning approaches (e.g.,\cite{prentice2009belief,platt2010belief,van2012motion,patil2015scaling,sun2015stochastic,sun2015high}) is also related to our work. But most of belief space planning approaches assume that the uncertainty is known, e.g., the form of the stochastic dynamics  are fully known. We do not even assume the form of the uncertainty is known here and we utilize Gaussian Process to keep tracking the uncertainty in a online manner while the robot is moving. 

Our work is closely related to that of Dey et al., \cite{dey2014gauss} who combined Canadian Traveler Problem (CTP) with GPs to formulate the problem of replanning as a Gaussian Traveler Problem (GTP). GTPs used determinization techniques like hindsight optimization \cite{olsen2011pon} to efficiently incorporate uncertainty over all edge costs of the graph in a GTP and show lower empirical cost of traversal to goal for an aircraft navigating partially known wind fields over continental US than merely replanning by the mean prediction over edge costs. GTPs have a number of limitations: 1. Discretization effects due to representing the problem on a graph. Making the graph dense has negative computational effects.  2. The edges of the graph may not be dynamically feasible for the aircraft to track. 3. The hindsight optimization determinization step requires sampling large number of possible future graph states which can be expensive. UCB-Replanning mitigates all of these issues.

\noindent \textbf{Multi-Armed Bandits}: \emph{Optimism in the face of uncertainty} is a classic strategy for trading off between exploration and exploitation in many Multi-Armed Bandit problems \cite{auer2002finite,auer2002nonstochastic,srinivas2010gaussian,bubeck2012regret} and Reinforcement Learning (RL) problems \cite{brafman2002r,li2012sample}. The classic Upper Confidence Bound (UCB) algorithm for MAB \cite{auer2002finite,bubeck2012regret} maintains a confidence interval of the true reward for each arm and pull the arm with the maximum upper bound of its confidence interval. UCB-Replanning leverages the classic analysis of MAB to analyze the performance of online receding horizon based planning.

\section{PRELIMINARIES}
Let us define $\mathbb{X}\subset \mathbb{R}^d$ as the state space for the robot. The state $\mathbf{x}\in \mathbb{X}$ includes the information of the robot such as positions and velocities. We model the uncertainty in the environment by a random variable $v\in\mathbb{R}$.\footnote{It is straightforward to extend to multi-variable case where variables can be modeled by multiple independent GPs} 
The realization of random variable $v$ depends on state of the robot and is modeled by an \emph{unknown} function subject to noise:
\begin{align}
v  = g(\mathbf{x}) + \epsilon,
\end{align} where we assume $\epsilon\sim\mathcal{N}(0, \sigma^2)$. These random variable could encode the variant types of uncertainties in the environment such as the speed of the wind at the current position of the robot, the estimated distance to a obstacle and so on. For notation simplicity, in the rest of the work, we define $v(\mathbf{x})$ as a (noisy) realization of the random variable $v$ at state $\mathbf{x}$.  Throughout this work, we assume that the unknown $g$ is sampled from a Gaussian process prior GP$(0, \kappa(\mathbf{x},\mathbf{x}'))$. 
Given a set of pairs $\{\mathbf{x}_i, v(\mathbf{x}_i)\}_{i=1}^N$, the posterior over $g$ is GP distribution with mean $\mu(\mathbf{x})$, covariance $cov(\mathbf{x},\mathbf{x}')$ and variance $\sigma^2(\mathbf{x})$ as:
\begin{align}
&cov(\mathbf{x},\mathbf{x}') = \kappa(\mathbf{x}, \mathbf{x}') - \mathbf{\kappa}_N(\mathbf{x})^T(\mathbf{K}_N+\sigma^2\mathbf{I})^{-1}\mathbf{\kappa}_N(\mathbf{x'}), \nonumber\\
&\mu(\mathbf{x}) = \mathbf{\kappa}_N(\mathbf{x})^T(\mathbf{K}_N + \sigma^2\mathbf{I})^{-1}\mathbf{y}_N, \;\; 
 \sigma^2(\mathbf{x}) = cov(\mathbf{x},\mathbf{x}),\nonumber
\end{align} where $\kappa_N(\mathbf{x}) = [\kappa(\mathbf{x}_1,\mathbf{x}), ..., \kappa(\mathbf{x}_N,\mathbf{x})]^T $ and $\mathbf{K}_N$ is the gram matrix with $\mathbf{K}_N [i,j] = \kappa(\mathbf{x}_i,\mathbf{x}_j)$. 

\begin{figure}[t!]
  \centering
      \includegraphics[width=0.3\textwidth]{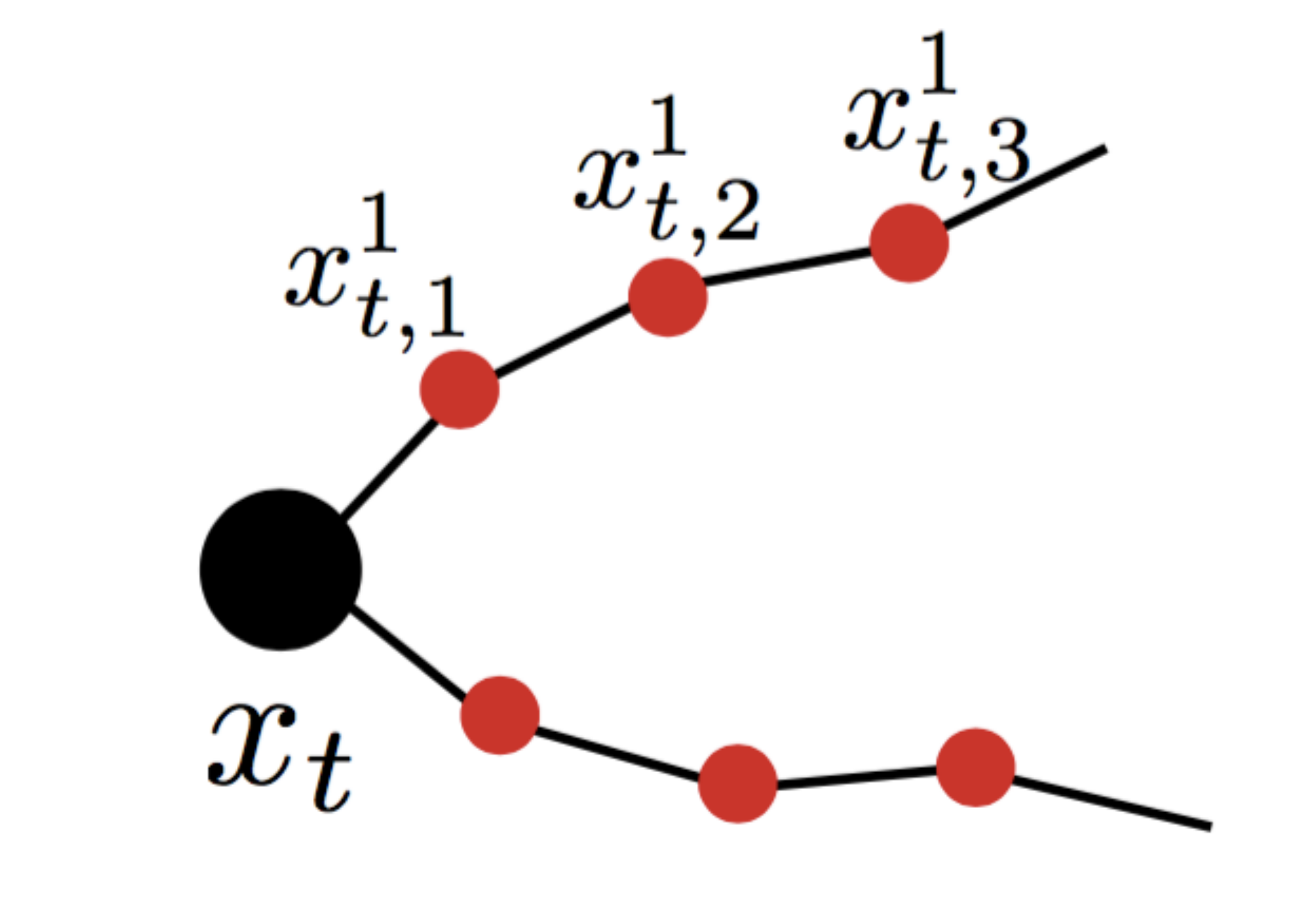}
  \caption{Notation of waypoints (red circles) of a set of trajectory library. At time step $t$ the robot is located at the root (black circle) of the trajectories and needs to make a decision about which trajectory to traverse next.}
  \label{fig:notations}
\end{figure}

We assume that the robot is equipped with a pre-computed library of trajectories $\{\tau_1,...,\tau_K\}$, where each trajectory $\tau_i$ consists of $L$ segments (Fig.~\ref{fig:notations}). Assume that at time step $t$, the robot's state is denoted as $\mathbf{x}_t$. Starting at $\mathbf{x}_t$, for each trajectory $\tau_i$, the $L+1$ waypoints are represented as $\{\mathbf{x}_{t,0}^i, \mathbf{x}_{t,1}^i, ..., \mathbf{x}_{t,L}^i\}$, where $\mathbf{x}_{t,0}^i = \mathbf{x}_t$, for all $i\in \{1,2,...,K\}$ (Fig.~\ref{fig:notations}).

At step $t$, located at state $\mathbf{x}_t$, the robot needs to pick a trajectory indexed at $I_t\in\{1,2.,...,K\}$  and execute $\tau_{I_t}$. Then the robot will traverse along $\tau_{I_t}$ and end at state $\mathbf{x}_{t,L}^{I_t}$. At the beginning of next around $t+1$, we set $\mathbf{x}_{t+1} = \mathbf{x}_{t,L}^{I_t}$, and repeat the above process. At every step $t$, each trajectory $\tau_i$ is equipped with a reward function $f_{t,i}$, which measures the reward of executing trajectory $\tau_i$ at step $t$. The reward function depends on the uncertain variables $v$ along the trajectory and we denote $f_{t,I_t}(\{v(\mathbf{x}_{t,j}^{I_t})\}_{j=0}^{L}): \mathbb{R}^{L+1}\rightarrow \mathbb{R}$.  Throughout this paper, we assume that $f_{t,i}$ is Lipschitz constant with respect to $\ell_1$ norm with Lipschitz constant $l$.

The ideal goal of the robot is to pick a sequence of trajectories $\{I^*_1,...,I^*_T\}$ from iteration $t=1$ to $T$, so that it can maximize the cumulative reward $\sum_{t=1}^T f_{t,I_t}(\{g(\mathbf{x}_{t,j}^{I_t})\}_{j=0}^{L})$ (Here we focus on maximizing with respect to $g(\mathbf{x})$, which is the expectation of $v(\mathbf{x})$).  Note that computing the optimal sequence of decisions $\{I_1^*, ..., I_T^*\}$ requires the full knowledge of the underlying map $g$, which is not available. 

It is not easy for the robot to pre-plan a sequence of the decisions since the robot does not have the exact information about $g$ except a prior, which could be non-informative. To refine its knowledge about $g(\mathbf{x})$, the robot needs to explore the area near $\mathbf{x}$ to collect observations of $v$ and update the GP.  Hence the robot needs to plan on the fly while collecting new information and refining its knowledge about $g$ for future planning. The robot essentially faces the tradeoff between exploration and exploitation: the robot needs to explore by choosing difference trajectories to get the information about the uncertain variables at different regions of its state space while it also needs to exploit by picking temporally high-reward trajectories to maximize its total reward.

\section{ALGORITHM}
We  leverage the strategy of optimism in the face of uncertainty to design our algorithm for robot to perform online replanning. Especially, we design an algorithm that is similar to UCB, where we maintain a confidence interval of the true reward for each trajectory. To design the confidence interval for each trajectory's reward at every step, we first extract a confidence interval of the uncertain variable $v$ from GP. We then use the Lipschitz continuity  of the reward function of each trajectory to transfer the confidence interval of the uncertain variable $v$ to the confidence interval of the reward of each trajectory. We finally choose the trajectory with the highest upper confidence bound of the reward estimation. The detailed algorithm is presented in Alg.~\ref{alg:ucb_replanning}.
\begin{algorithm}[tb]
\caption{UCB-Replanning}
 \label{alg:ucb_replanning}
\begin{algorithmic}[1]
  \STATE {\bfseries Input:} A library of $K$ trajectories $\{\tau_k\}_{k=1}^K$, sequence of parameters $\beta_t\in\mathbb{R}^+$. A GP $(\mu_0, \sigma_0)$ that models the variable $v$ over the state space $\mathbb{X}$.
 \FOR {t = 1 to T}
    \FOR {k = 1 to K}
        \STATE Compute the sequence of means of $g$ on the waypoints on $\tau_k$ as $\{\mu_{t-1}(\mathbf{x}_{t,j}^k)\}_{j=0}^{L}$\label{line:gp_mean}.
        \STATE Compute the sequence of standard deviations of $g$ on the waypoints as $\{\sigma_{t-1}(\mathbf{x}_{t,j}^k)\}_{j=0}^{L}$.
        \label{line:gp_sd}
        \STATE Compute the upper confidence bound of the reward function $f_{t,k}$ as: \label{line:ucb_equation}
            $b_k := f_{t,k}(\{\mu_{t-1}(\mathbf{x}_{t,j}^k)\}_{j=0}^L)+l\beta_t^{1/2}\sum_{j=0}^{L}\sigma_{t-1}(\mathbf{x}_{t,j}^k)$. 
    \ENDFOR
    \STATE Choose index $I_t = \arg\max_{k\in\{1,...,K\}} b_{k}$ and execute trajectory $\tau_{I_t}$.
    \STATE Observe samples of $g$ along the waypoints as $\{v(\mathbf{x}_{t,j}^{I_t})\}_{j=0}^{L-1}$ and use these $L$ samples to online update GP to obtain $\mu_t$ and $\sigma_t$. 
    \label{line:gp_update}
 \ENDFOR
 \end{algorithmic}
\end{algorithm}
In round $t$, Alg.~\ref{alg:ucb_replanning} first use the current GP model $(\mu_{t-1}, \sigma_{t-1})$ to compute the means and the standard deviations of $v$ along the waypoints of all $K$ trajectories (Line.~\ref{line:gp_mean} and \ref{line:gp_sd}). Then for each trajectory $\tau_k$, Alg.~\ref{alg:ucb_replanning} using the Lipschitiz constant $(l)$ and a scaling parameter $\beta_t$ (will be defined later in analysis) to compute the upper confidence bound of the reward function as shown in Line~\ref{line:ucb_equation}. It then picks the trajectory $\tau_{I_t}$ that has the highest upper confidence bound. During the execution of $\tau_{I_t}$, the robot receives observations of $v$ along the waypoints and online updates the GP model  (Line.~\ref{line:gp_update}).

\subsection{Analysis} 
We analyze the performance of Alg.~\ref{alg:ucb_replanning}. Particularly, we are interested in analyzing the \emph{regret}, which measures the difference between Alg.~\ref{alg:ucb_replanning}'s cumulative reward and the cumulative reward if one always picks the best trajectories along the states that the robot traversed when executing Alg.~\ref{alg:ucb_replanning}. More formally, let us assume that the the sequence of states that the robot visited at all $T$ rounds as: $\{\mathbf{x}_{1}, \mathbf{x}_{2}, ..., \mathbf{x}_{T}\}$ and the indexes of the trajectories that the robot picked at all $T$ rounds as $\{I_1,...,I_T\}$. We define the regret as:
\begin{align}
\label{eq:regret_definition}
&\mathbf{R}_T = \frac{1}{T}\Big[\sum_{t=1}^T\max_{I_t^*\in[K]}f_{t,I_t^*}\big(\{g(\mathbf{x}_{t,j}^{I_t^*}) \}_{j=0}^{L}\big) \nonumber \\
&\;\;\;\; - \sum_{t=1}^T f_{t,I_t}\big(\{g(\mathbf{x}_{t,j}^{I_t})\}_{j=0}^{L}\big) \Big]
\end{align} Namely, at each round $t$, we measure how much more reward the robot could gain if it could pick $I_t^*$ instead of $I_t$ at $\mathbf{x}_t$. The goal is to make regret converges to zero so that in average the robot has little regret in terms of choosing $I_t$. 

We remark that our regret definition measures the difference between the rewards of an optimal decision maker with full access to latent information and the rewards of the learning algorithm \emph{along the states taken by the learning algorithm}. 
Ideally one would be interested in the regret of the learning algorithm in respect to the rewards of the optimal decision maker \emph{along the states generated from the optimal decision maker itself}. It turns out that the latter definition of regret is impossible to achieve without any assumptions about the reachability of the systems and the ability to reset \cite{li2012sample}.  Consider the MDP shown in Fig.~\ref{fig:mdp_example} with 3 states and 2 actions. Once the agent makes the mistake of taking action $a_2$ at $x_0$, possibly due to the lack of the full knowledge of the model,
it will be stuck in $x_2$ forever and the regret with
respect to the optimal decision maker on the optimal sequence $\{x_0, x_1, x_1,...\}$ will grow linearly.
It is also worth mentioning that our definition of regret in
Eqn.~\ref{eq:regret_definition} is similar to the classic sample complexity definition
of exploration in reinforcement learning \cite{li2012sample}, in the sense that
the sample complexity of exploration measures the number
of mistakes the learning algorithm makes on the sequence of
states generated from the algorithm, instead of the sequence of the states resulting from the optimal policy. 

\begin{figure}[t]
  \centering
      \includegraphics[width=0.3\textwidth]{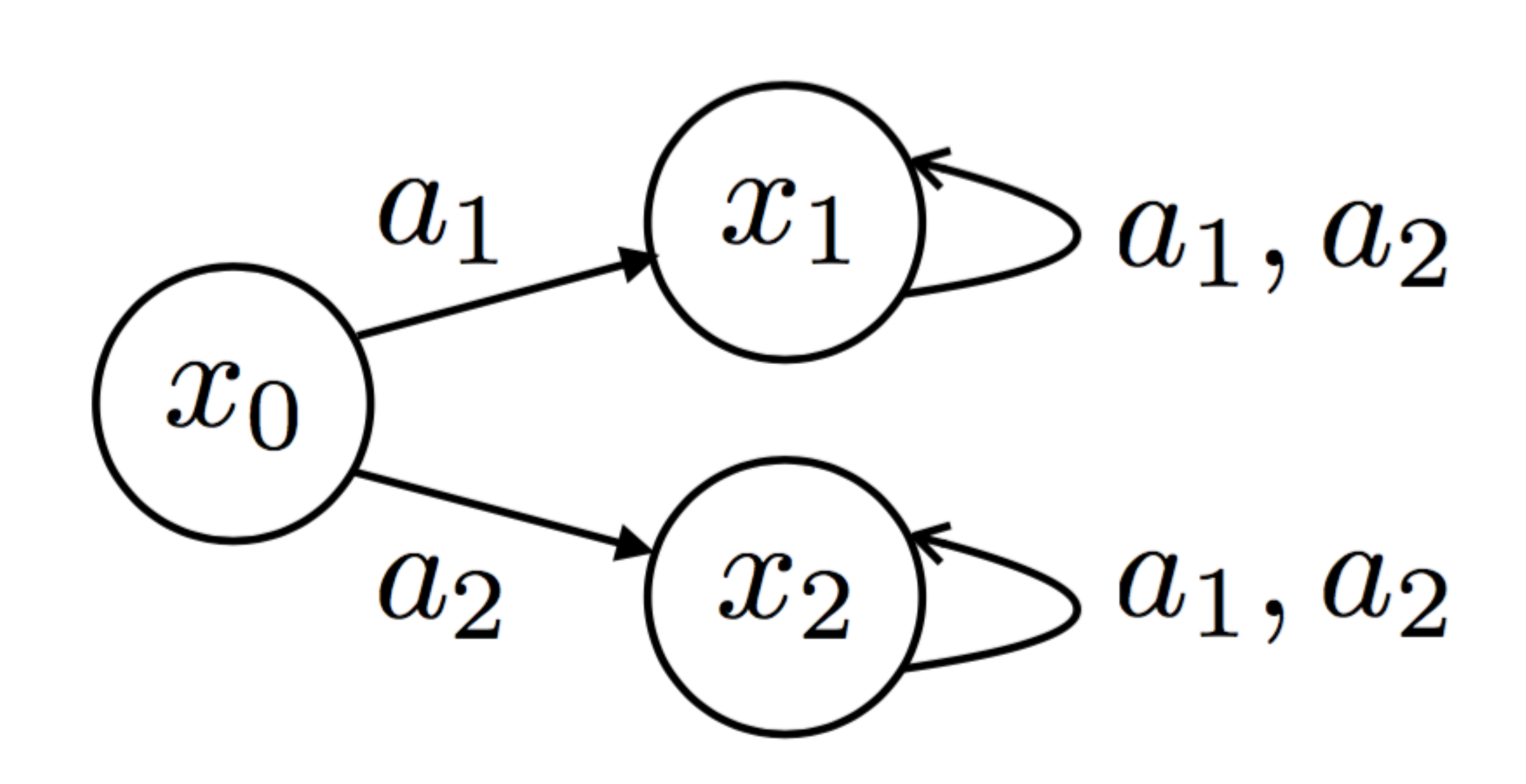}
  \caption{A decision making problem with three states and two possible actions. The agent starts from x0 and upon taking action $a_1$ ($a_2$) reaches $x_1$ ($x_2$). The reward structure is such that the agent receives a reward $1$ upon landing at state $x_1$ and no reward elsewhere. Once the agent lands in either $x_1$ or $x_2$, there is no way for it to move to any other state. Hence the optimal sequence of state for this problem is  $\{x_0, x_1, x_1,...\}$.
  }
  \label{fig:mdp_example}
  \vspace{-10pt}
\end{figure}

Now we are ready to show that Alg.~\ref{alg:ucb_replanning} is no-regret: $\mathbf{R}_t\to 0$, as $T\to\infty$. Let us define $D_t$ as the states of all waypoints on all $K$ trajectories: $D_t = \{\{\mathbf{x}_{t,j}^{k}\}_{j=0}^{L}\}_{k=1}^{K}$, where $|D_t| = (L+1)K$. The following lemma builds the confidence interval over $g$ on all waypoints in $D_t$ at all rounds $t$:
\begin{lemma}
\label{lemma:high_prob_on_variable}
With $\beta_t = 2\log(\frac{|D_t|\pi_t}{\delta})$ and any $\pi_t$ that satisfies $\sum_t 1/\pi_t = 1, \pi_t >0$, with probability at least $1-\delta$ we have:
\begin{align}
\forall t, \forall x\in D_t,  |g(\mathbf{x}) - \mu_{t-1}(\mathbf{x})|\leq \beta_t^{1/2}\sigma_{t-1}(x).
\end{align}
\end{lemma}

The above lemma is essentially the same as Lemma 5.1 in \cite{srinivas2010gaussian}. For completeness we include the proof in the appendix. 


The next lemma builds a confidence interval over the rewards of all $K$ trajectories, at all rounds. 
\begin{lemma}
Set $\beta_t = 2\log(\frac{|D_t|\pi_t}{\delta})$ and $\sum_t 1/\pi_t = 1, \pi_t >0$, we have with probability at least $1-\delta$:
\begin{align}
\label{eq:lipschitz_f}
&\forall k\in [K],\forall t, |f_{t,k}(\{g(\mathbf{x}_{t,j}^k)\}_{j=0}^{L}) - f_{t,k}(\{\mu_{t-1}(\mathbf{x}_{t,j}^k)\}_{j=0}^{L})| \nonumber\\
&\;\;\;\;\; \leq l\beta_t^{1/2} \sum_{j=0}^{L}\sigma_{t-1}(\mathbf{x}_{t,j}^{k}). 
\end{align}
\label{lemma:ci_traj}
\end{lemma}
\begin{proof}
Let us define event $A$ as $\forall t, \forall x\in D_t,  |g(\mathbf{x}) - \mu_{t-1}(\mathbf{x})|\leq \beta_t^{1/2}\sigma_{t-1}(\mathbf{x})$ and from the previous lemma, we know that event $A$ happens with probability at least $1-\delta$. We now condition on event $A$.  Below we show that event $A$ will imply the inequality in the above lemma. 

Since we assume that the reward functions $f_{t,k}$ is Lipschitz continuous, we must have for any $t$ and $k$:
\begin{align}
&|f_{t,k}(\{g(\mathbf{x}_{t,j}^k)\}_{j=0}^{L}) - f_{t,k}(\{\mu_{t-1}(\mathbf{x}_{t,j}^k)\}_{j=0}^{L})| \nonumber\\
&\;\;\;\; \leq l\sum_{j=0}^{L}|g(\mathbf{x}_{t,j}^k) - \mu_{t-1}(\mathbf{x}_{t,j}^k)|, \nonumber
\end{align} where $l$ is the Lipschitz constant. Conditioned on the fact that event $A$ happens, we have that at any round $t$ and for any trajectory $\tau_k$, we have that:
\begin{align}
&|f_{t,k}(\{g(\mathbf{x}_{t,j}^k)\}_{j=0}^{L}) - f_{t,k}(\{\mu_{t-1}(\mathbf{x}_{t,j}^k)\}_{j=0}^{L})| \nonumber\\
& \leq l\sum_{j=0}^{L}|g(\mathbf{x}_{t,j}^k) - \mu_{t-1}(\mathbf{x}_{t,j}^k)| \leq l\sum_{j=0}^{L}\beta_t^{1/2}\sigma_{t-1}(\mathbf{x}_{t,j}^k), \nonumber
\end{align}where we use the assumption that $f$ is $l$-Lipschitz continuous in $\ell_1$ norm. To this end, we have already shown that event $A$ implies the inequality in the above lemma. 
Since the probability that the event $A$ happens is at least $1-\delta$, we must have with probability at least $1-\delta$, $\forall t, k$ 
\begin{align}
&|f_{t,k}(\{g(\mathbf{x}_{t,j}^k)\}_{j=0}^{L}) - f_{t,k}(\{\mu_{t-1}(\mathbf{x}_{t,j}^k)\}_{j=0}^{L})| \nonumber\\
&\leq l\sum_{j=0}^{L}\beta_t^{1/2}\sigma_{t-1}(\mathbf{x}_{t,j}^k).
\end{align}Hence we prove the lemma. 
\end{proof}

Now we present the main theorem. We consider two types of Kernels: (1) linear kernel as $\kappa(\mathbf{x},\mathbf{x}') = \mathbf{x}^T\mathbf{x}'$, and (2) square exponential kernel as $\kappa(\mathbf{x},\mathbf{x}') = \exp(-c\|\mathbf{x}-\mathbf{x}'\|^2)$. 
\begin{theorem}
\label{them:main_them}
With $\beta_t = 2\log(\frac{|D_t|\pi_t}{\delta})$ and $\sum_t 1/\pi_t = 1, \pi_t >0$, we will have with probability at least $1-\delta$:
\begin{align}
\mathbf{R}_T/T \leq O(d\log(LT)/T) \rightarrow 0,\;\;\; T\rightarrow\infty,
\end{align} for linear kernel as $\kappa(\mathbf{x},\mathbf{x}') = \mathbf{x}^T\mathbf{x}'$, and
\begin{align}
\mathbf{R}_T/T\leq O\big((\log (LT) )^{d+1} / T\big)\rightarrow 0,\;\;\; T\rightarrow\infty,
\end{align} for squared exponential kernel $\kappa(\mathbf{x},\mathbf{x}') = \exp(-c\|\mathbf{x}-\mathbf{x}'\|^2)$.
\end{theorem}
\begin{myproof}[Proof Sketch] For the sake of brevity, we provide the complete proof in the appendix and an
abbreviated sketch below. Let us define event $B$ as the inequality~\ref{eq:lipschitz_f} shown in Lemma~\ref{lemma:ci_traj}. 
From Lemma.~\ref{lemma:ci_traj} we know that the probability of event $B$ happens is at least $1-\delta$. Below we show that event $B$ implies the two equalities in the above theorem. For the rest of the proof, we assume we condition on that event $B$ happens. Consider round $t$. Note that $I_t$ is defined as:
\begin{align}
\label{eq:m_1_1}
& I_t = \arg\max_{k\in[K]}f_{t,k}(\{\mu_{t-1}(\mathbf{x}_{t,j}^k)\}_{j=0}^{L}) \nonumber\\
&\;\;\;\; + l\beta^{1/2}\sum_{j=0}^{L}\sigma_{t-1}(\mathbf{x}_{t,j}^{k}),  
\end{align}and $I_t^*$ is defined as:
\begin{align}
\label{eq:m_2_2}
I_t^* = \arg\max_{k\in[K]}f_{t,k}(\{g(\mathbf{x}_{t,j}^k)\}_{j=0}^{L}),
\end{align} namely the best trajectory one would pick at this round $t$ if $g$ is known. Now let us define the single step regret $r_t$ as:
\begin{align}
r_t = f_{t,I_t^*}\big(\{g(\mathbf{x}_{t,j}^{I_t^*})\}_{j=0}^{L}\big) - f_{t,I_t}\big(\{g(\mathbf{x}_{t,j}^{I_t})\}_{j=0}^{L}\big),
\end{align}namely the regret  one has by choosing $I_t$ instead of $I_t^*$ at round $t$. Similar to classic analysis of UCB based algorithms, we can upper bound $r_t$ using Eqn.~\ref{eq:m_1_1} and \ref{eq:m_2_2}:
\begin{align}
& r_t \leq 2l\beta_t^{1/2}\sum_{j=0}^{L}\sigma_{t-1}(\mathbf{x}_{t,j}^{I_t}).
\end{align}
Square both sides of the above inequality and use similar techniques from \cite{srinivas2010gaussian}, we have for $r_t^2$:
\begin{align}
&r_t^2 = 4l^2\beta_t(\sum_{j=0}^{L}\sigma_{t-1}(\mathbf{x}_{t,j}^{I_t}))^2\leq 4l^2\beta_tL\sum_{j=0}^{L}\sigma_{t-1}(\mathbf{x}_{t,j}^{I_t})^2 \nonumber \\
&\leq 4l^2\beta_TL\sigma^2 C_1\sum_{j=0}^{L}\log\big(1+\sigma^{-2}\sigma_{t-1}(\mathbf{x}_{t,j}^{I_t})^2\big)
\end{align} where $C_1 = \sigma^{-2}/\log(1+\sigma^{-2})\geq 1$.

Since the regret $\mathbf{R}_T = \sum_{t=1}^T r_t$, we must have $\mathbf{R}_T^2\leq T\sum_t r_t^2$. Using Lemma 5.3 and Lemma 5.4 from \cite{srinivas2010gaussian}, we can link $\mathbf{R}_T$ to the maximum information gain as follows:
\begin{align}
&\mathbf{R}_T^2 \leq 4l^2\beta_TL^2\sigma^2C_1 T\sum_{t=1}^T\sum_{j=0}^{L-1}\log(1+\sigma^{-2}\sigma_{t-1}(\mathbf{x}_{t,j}^{I_t})^2) \nonumber\\
&\;\;\;\;\;\;  \leq 4l^2\beta_TL^2\sigma^2C_1 T \gamma_T,   \nonumber
\end{align} where $\gamma_T$ is the maximum information gain defined as 
\begin{align}
&\gamma_T = \max_{A\subseteq \mathbb{X}, |A| = LT}I(v_A; g) \nonumber\\
&= \max_{A\subseteq \mathbb{X}, |A| = LT}H(v_A) - H(v_A | g),\nonumber
\end{align} where $H(x)$ is the entropy of the random variable $x$, $H(x|y)$ is the conditional entropy, $v_A = \{g(\mathbf{x}) + \epsilon\}_{x\in A}$ is the set of observations of $g(\mathbf{x})$ for all states $\mathbf{x}$ in set $A$. Namely $\gamma_T$ quantifies the maximum reduction in uncertainty about $g$
from revealing the observations of $g$ on $LT$ states.

Theorem 5 from \cite{srinivas2010gaussian} shows $\gamma_T\leq O(d \log LT)$ if $\kappa(\mathbf{x},\mathbf{x}') = \mathbf{x}^T\mathbf{x}'$ and $\gamma_T\leq O((\log LT)^{d+1})$ if $\kappa(\mathbf{x},\mathbf{x}') = \exp(-c\|\mathbf{x}-\mathbf{x}'\|^2)$. Substitute these results to the above inequality, we prove the theorem. 
\end{myproof}

The above theorem shows that as the number of rounds approaches infinity, in average, the policy presented at Alg.~\ref{alg:ucb_replanning} performs almost as well as the \emph{best} policy which can always choose the best trajectory at every round. Note that the average regret of using squared exponential kernel (e.g., RBF kernel) shrinks more slowly than the average regret of linear kernel. This indicates that in general if the wind map $g$ is complicated (i.e., requiring RBF kernel to model it), Alg.~\ref{alg:ucb_replanning} requires more rounds to achieve good performance.

\section{CASE STUDY: AIRCRAFT NAVIGATION UNDER WIND UNCERTAINTY}
\label{sec:case_study}
We conduct a case study of aircraft nagiviation under wind uncertainty and show how UCB-Replanning can be applied. Let us define the state space of the aircraft $\mathbb{X}\in\mathbb{R}^2$ (2D position) where we assume that $\mathbb{X}$ is compact. We fix the norm of airplane's speed to $v_0\in\mathbb{R}^2$. At a particular position $\mathbf{x}$, the speed of the wind $v\in\mathbb{R}^2$ is computed from an unknown mapping $g:\mathbb{X}\to \mathbb{R}^2$, subject to Gaussian noise as $v(\mathbf{x}) = g(\mathbf{x}) + \epsilon$, $\epsilon \sim \mathcal{N}(0,\sigma^2 I)$. We assume that range of the speed of wind is bound as $\|v\|_2 \in [v_{\min}, v_{\max}]$, where $v_{\min},v_{\max}\in\mathbb{R}^+$, and we further assume that $v_{\max} = \|v_0\|_2 / 2$. Give the position $\mathbf{x}$ of the aircraft, the real speed of the aircraft can be computed as follows:
\begin{align}
\tilde{v}(\mathbf{x}) = v_0 + \frac{\langle v_0, v(\mathbf{x})  \rangle}{\|v_0\|_2^2}v_0,
\end{align} where the second part of the RHS of the above equation is the projection of the wind speed $v(\mathbf{x})$ at location $\mathbf{x}$ onto the airplane's speed. Overall the goal of the aircraft is to leverage the speed of the wind to decrease its traveling time. 

We use Gaussian Process with squared exponential kernel $\kappa(\mathbf{x},\mathbf{x}') = \alpha \exp(-c\|\mathbf{x}-\mathbf{x}'\|^2)$ ($c,\alpha\in \mathbb{R}^+$) to model the wind speed.  At every round, the aircraft chooses a trajectory from a pre-computed library of $K = 48$ trajectories \cite{green2007toward}. 
Each trajectory is a spline with $L$ segments, each segment having a length $d$.

Given the start position and the final position, the goal of the aircraft is to minimize the total traveling time. Hence our reward function is related to the traveling time. Let us assume that at round $t$, the robot is at state $\mathbf{x}_t$.  For each trajectory $\tau_k, k\in [K]$, we design the reward function with respect to the speed of the wind as:
\begin{align}
\label{eq:reward}
f_{t,k}(\{v(\mathbf{x}_{t,j}^{k})\}_{j=0}^{L}) = -\Big(\sum_{i=0}^{L-1}\frac{d}{\|\tilde{v}(\mathbf{x}_{t,i}^{k})\|_2} + \lambda_t t_{g}\Big),
\end{align} where $\lambda_t\in (0,1]$.  The first part of the RHS of the above equation is the total time for traversing the trajectory $\tau_k$ while the second part $t_g$ serves as an estimation of the left time for traveling from the the end of the trajectory $\tau_k$ and the final position. In this work, we use the total time of traveling along the Great Circle Route for $t_g$ (Note that $t_g$ could also be regarded as a scaled shortest distance to goal).  \footnote{We experimentally verified that incorporating wind estimation into the computation of $t_g$ significantly worsen the performance. This is because the wind estimation around the area that is far away from the airplane's current position is usually low-quality.}

To apply UCB-Replanning, we first set $\pi_t = \pi^2 t^2 / 6$. It is straightforward to compute the Lipschitz continuous constant $l$ of $f_{t,k}$ as $l = \frac{4d}{\|v_0\|_2^2}$. Hence we can set $l\beta_t^{1/2} = O\Big(\frac{4d}{\|v_0\|_2^2}\sqrt{\log(\frac{KL\pi^2 t^2/6}{\delta})} \Big)$. Throughout the experiments, we set $\delta = 0.05$. Namely we want the inequalities in Theorem~\ref{them:main_them} to hold with probability at least 95\%. For specific values of $l\beta_t^{1/2}$, we set $l\beta_t^{1/2} = c\frac{4d}{\|v_0\|_2^2}\sqrt{\log(\frac{KL\pi^2 t^2/6}{\delta})}$, where $c\in\mathbb{R}^+$. We test different values of $c$ in experiments. 

\begin{figure}[t]
	\centering
	\begin{subfigure}[l]{0.239\textwidth}
        \includegraphics[trim=0mm 30mm 0mm 50mm, clip, width=1.07\textwidth,keepaspectratio]{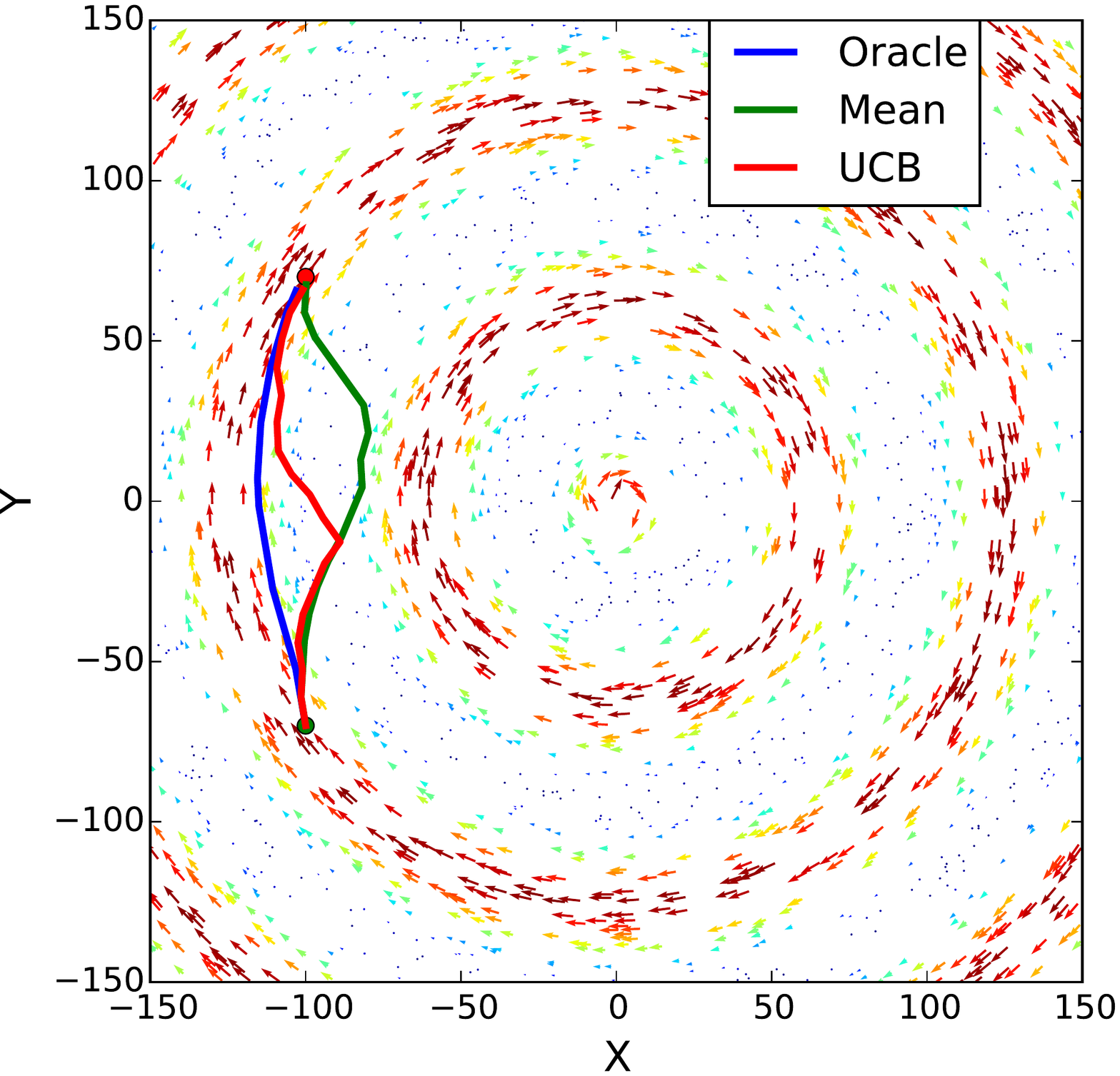}
         	\caption{Wind field 1 (Tail Wind)}
    \end{subfigure}
	\begin{subfigure}[l]{0.239\textwidth}
        \includegraphics[trim=0mm 30mm 0mm 50mm, clip, width=1.07\textwidth,keepaspectratio]{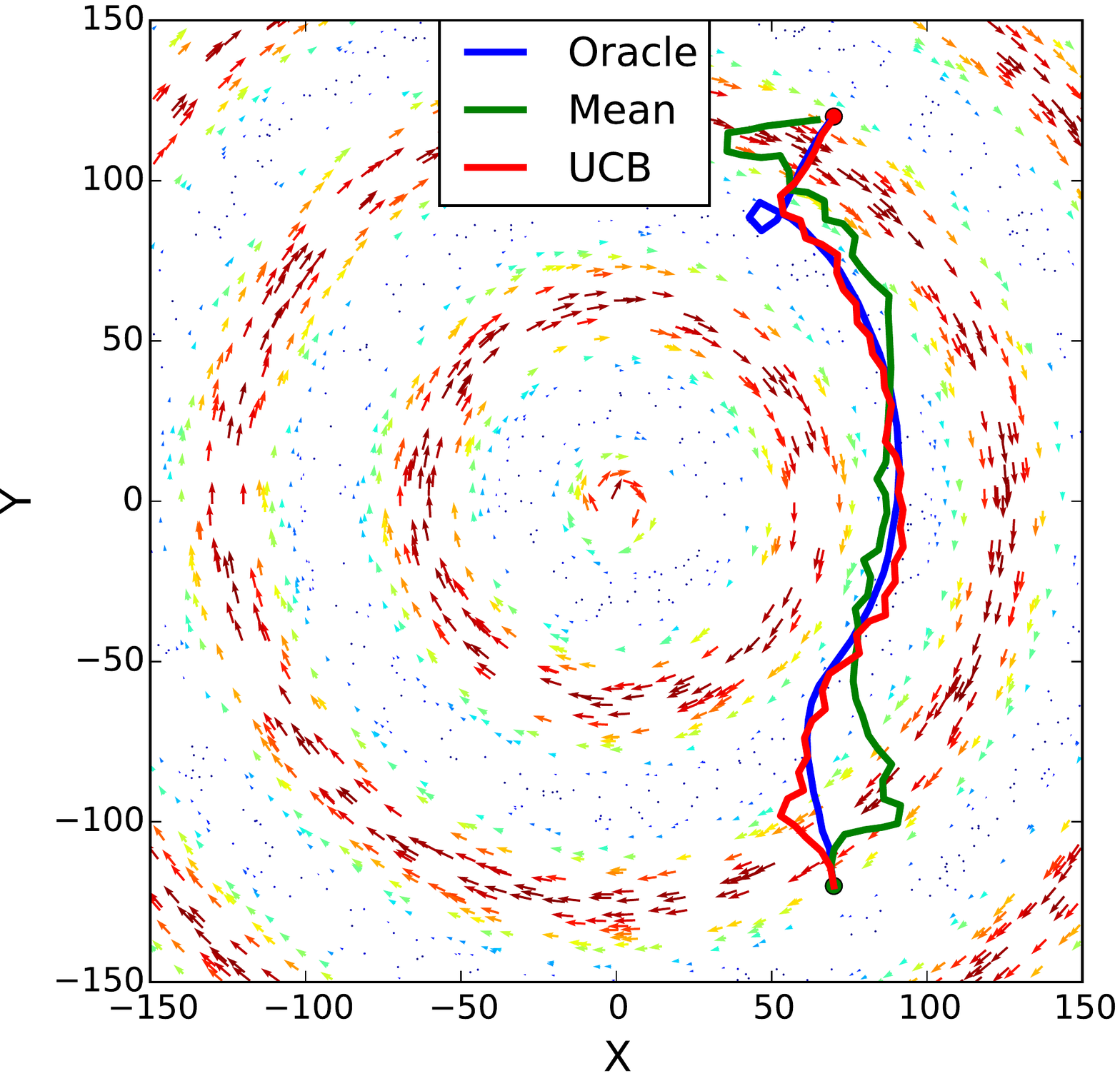}
    	\caption{Wind field 1 (Head Wind) }
    \end{subfigure}
    \caption{Trajectories (solid lines) resulting from Oracle, Mean and UCB, in the synthetic wind field, with different start (green dot) and goal position (red dot) settings
    }
     \label{fig:wind_field_trajectory}
\end{figure}

\section{EXPERIMENTS}
\begin{table}[t]
\begin{center}
\resizebox{0.35\textwidth}{!}{ 
    \begin{tabular}{llr}
    \toprule
     & Tail Wind & Head Wind \\ 
     \midrule
    \textbf{ Oracle} & 18.8\% &  88.2\% \\  
    \textbf{ UCB} & 5.4\% & 84.0\%  \\ 
    \textbf{ Mean} & 2.9\% & 75.2\%  \\ 
    \bottomrule
    \end{tabular}}
\end{center}
\vspace{-0pt}
\caption{Percentage improvement of Oracle, UCB and Mean compared to simply traveling along the straight line, under the synthetic wind filed setup. The percentage is computed as the difference between traveling time on straight line and traveling time of UCB (Oracle, Mean) divided by the traveling time on straight line. }
\label{tab:comparison_synthetic}
\end{table}

We test our  UCB-Replanning algorithm (\emph{UCB}) on the application of aircraft flight path planning with partially observed wind. We compare our algorithm to there baselines: (1) receding horizon based Oracle (\emph{Oracle}),  (2) receding horizon based Plan by Mean (\emph{Mean}) and (3) Great Circle Rout (\emph{GCR}), namely following the shortest path in geodesic distance. The  \emph{Oracle}, which has access to the true wind information (i.e., it knowns $g(\mathbf{x})$), picks trajectory using wind map $g(\cdot)$. Namely \emph{Oracle} replace $b_k$ in Alg.~\ref{alg:ucb_replanning} using $f_{t,k}(g(\mathbf{x}_{t,0}^k), ...,g(\mathbf{x}_{t,L}^k))$.  \emph{Mean} doesn't have access to the true wind information. Instead,  \emph{Mean} exactly follows the same structure of UCB-Replanning, but replace $b_k$ in Line~\ref{line:ucb_equation} of Alg.~\ref{alg:ucb_replanning} using $f_{t,k}(\mu_{t-1}(\mathbf{x}_{t,0}^k), ...,\mu_{t-1}(\mathbf{x}_{t,L}^k))$ (i.e., use the mean $\mu_{t-1}$ but ignore the standard deviations $\sigma_{t-1}$).

\subsection{Synthetic Wind Field}
We created a wind field as shown in Fig.~\ref{fig:wind_field_trajectory}. The arrow indicates the direction of the wind filed and the length of the arrows indicates the strength of the wind.  We pre-computed 25 trajectories, where each trajectory is simply a straight line with 30 segments. We set the length of each segment to be 0.2, the norm of the airplane's speed to 2.0, and the maximum norm of wind speed to 1.0. 

When the airplane is traveling downwind, the strategy is to leverage the stronger wind to shorten the traveling time, as show by the trajectory of  Oracle (blue) and the trajectory of  UCB (red) in Fig.~\ref{fig:wind_field_trajectory} (a).  When the airplane is traveling upwind, one strategy to save traveling time is to identify the regions where the wind is not strong, which is exactly what  Oracle,  UCB and  Mean performed in Fig.~\ref{fig:wind_field_trajectory} (b). 
Also as we can see from Fig.~\ref{fig:wind_field_trajectory},  UCB's trajectory  and  Oracle's trajectory are usually  different due to possible exploration at the beginning, but then gradually converges to each other. On the other hand,  Mean may perform quite sub-optimally as shown in Fig.~\ref{fig:wind_field_trajectory} (a).

Table~\ref{tab:comparison_synthetic} shows the percentage improvement of Oracle, UCB and Mean over GCR. The data used to compute the numbers in Table~\ref{tab:comparison_synthetic} is collected from 100 trials with different wind speed, and start/goal positions. As we can see,  UCB consistently outperforms  Mean, especially in the head wind case.  Oracle generally performs the best since it has access the true underlying wind field (not available in practice). In summary, the comparison clearly shows that the tradeoff in exploration and exploitation introduced by UCB strategy is beneficial for Receding Horizon Control, while pure exploitation based strategy (i.e., Mean) in some cases can perform sub-optimally.

\subsection{Real Wind Field}

\begin{figure}[t]
	\centering
	\begin{subfigure}[l]{0.4\textwidth}
        \includegraphics[trim=20mm 20mm 50mm 20mm, clip, width=1.0\textwidth,keepaspectratio]{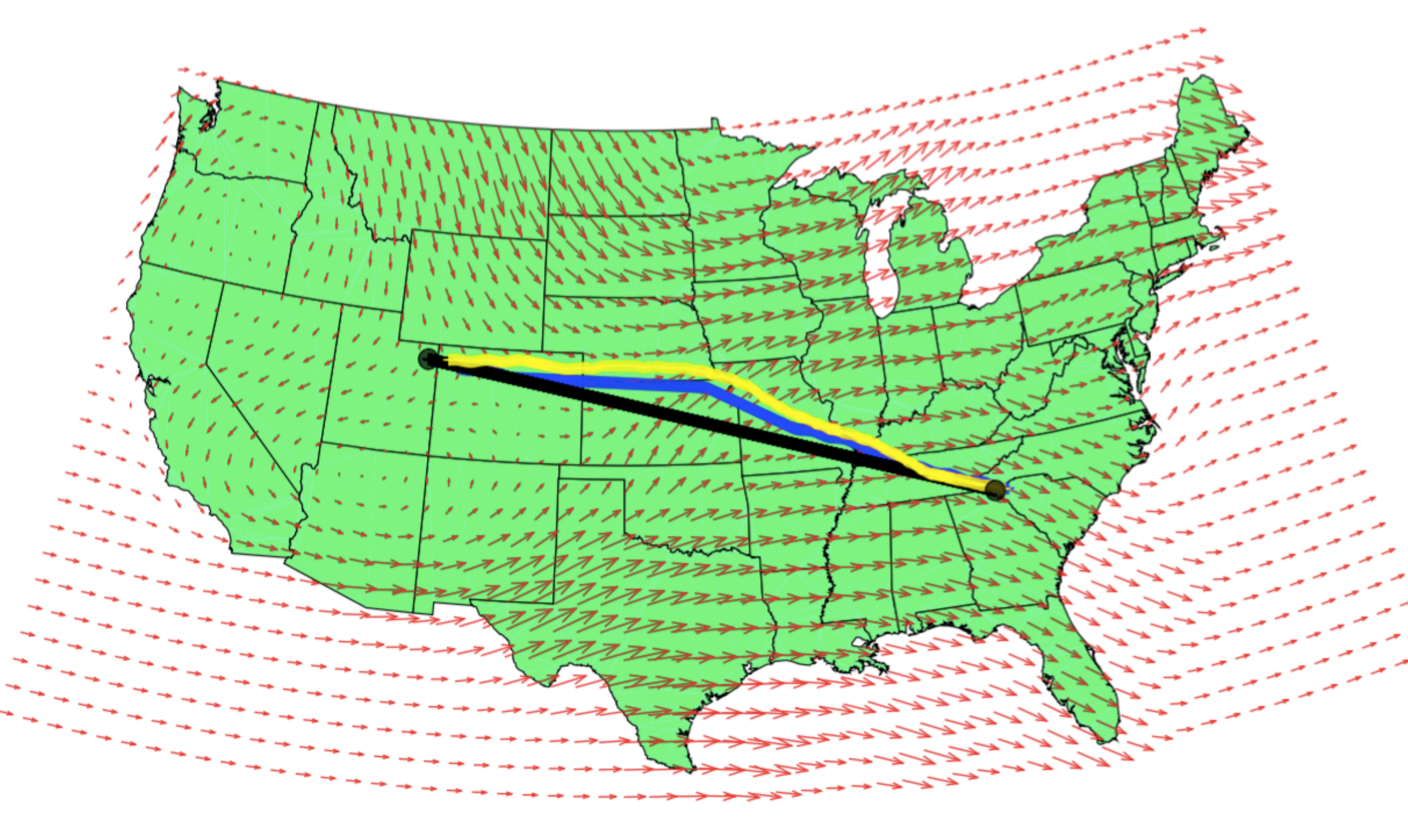}
    	\caption{South Carolina to Utah (Head Wind)}
    \end{subfigure}
  	\begin{subfigure}[l]{0.4\textwidth}
        \includegraphics[trim=20mm 20mm 50mm 20mm, clip, width=1.0\textwidth,keepaspectratio]{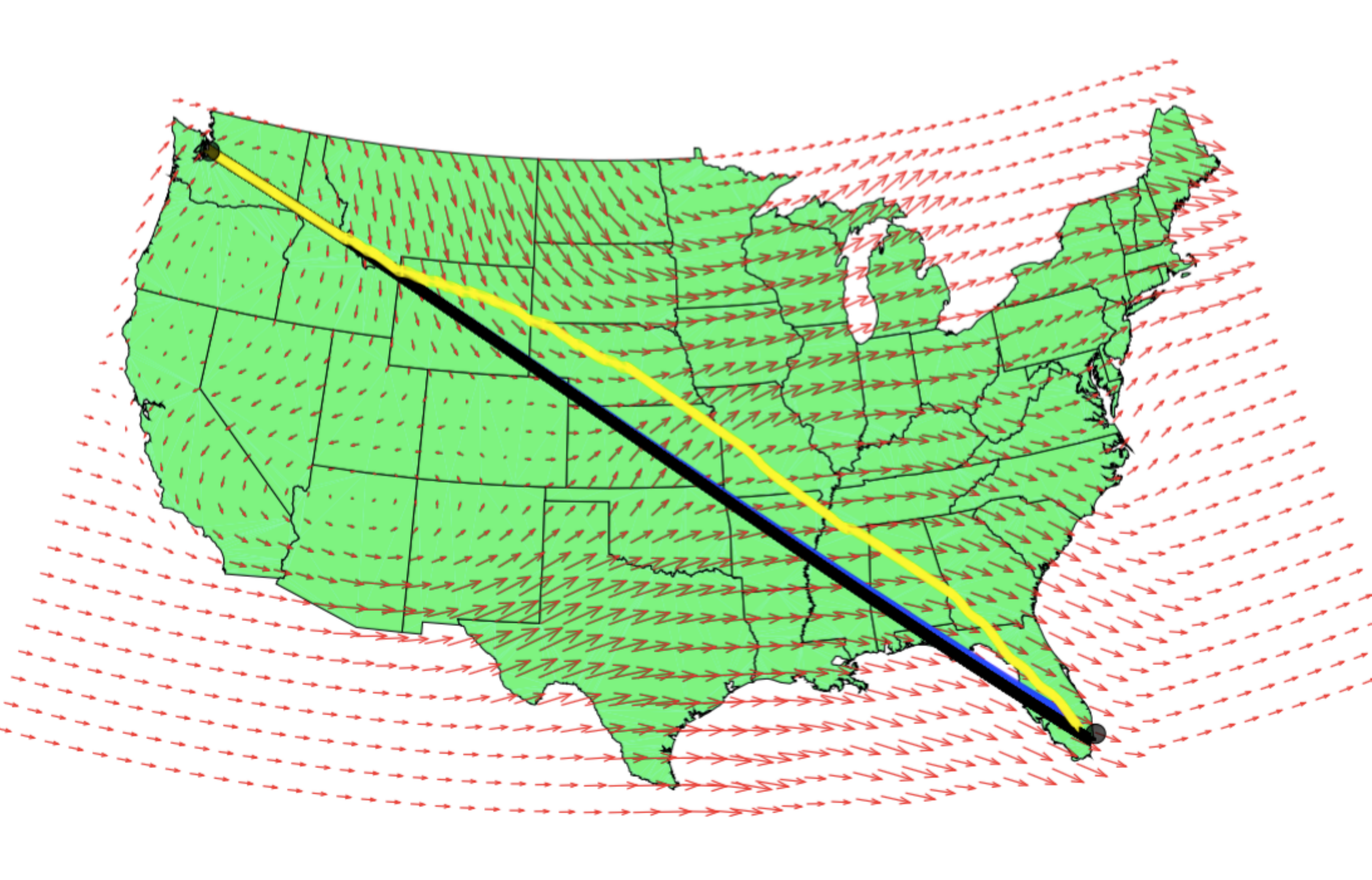}
    	\caption{Seattle to Miami (Tail Wind)}
    \end{subfigure}
        \caption{Examples of trajectories resulting from Mean (Yellow) and UCB (Blue) for a  short rout from South Caroline to Utah (a), and a long rout from Seattle to Miami (b). We also plot the Great Circle Rout (black).}
     \label{fig:real_wind_map}
\end{figure}

We also tested our algorithm on wind map constructed from real data.  We define boundaries to be the Continental United States. The Northwest boundary was set as (49.5N, 125.0W) and the Southeast as (25.0N, 67.5W). The simulated aircraft, maintains a constant cruising speed of 250 knots at an altitude of 39000 feet (11887 m). Winds encountered at this altitude could go upto upwards of 100 knots.  Using realistic data provided by NOAA we construct wind maps by fitting a Gaussian Process over wind data from all 176 stations that NOAA maintains (e.g., Fig.~\ref{fig:real_wind_map} shows one instance of generated routs). Since it is impossible to get true wind map over US, we simply use the mean of the fitted GP as an estimation of ground truth of the wind speed.  We refer readers to  \cite{kapoor2014airplanes} for details of wind map construction.

We use an existing pre-computed library of trajectories from \cite{green2007toward}. We test UCB and Mean on two different routes: (1) a \emph{short} route from South Carolina to Utah (around 1300 nautical miles),  and (2) a \emph{long} route from Seattle to Miami (around 2700 nautical miles). As we can see from Fig.~\ref{fig:real_wind_map} (a), when flying with head wind, both UCB (and Mean in this case) exhibits another strategy to save traveling time: it guides the aircraft to fly in the direction that is nearly perpendicular to the wind speed in order to cancel the wind effect when the wind is strong. When flying with good tail wind as shown in Fig.~\ref{fig:real_wind_map} (b)), UCB almost identifies the shortest path (Great Circle Rout) and follow it. Note that the trajectory resulting from UCB shown in Fig.~\ref{fig:real_wind_map} (b) is still a little bit different from GCR.

We simulate 11 days of real wind data by dividing each day into 6 hour time slots and simulating both paths for every slot (Fig.~\ref{fig:real_wind_map} shows one instance of the constructed wind maps). This in total give us 80 different trials for UCB and Mean, 40 for head wind and 40 for tail wind. We report the average traveling time and standard deviation in Table~\ref{tab:comparison_real}. As we can see, in average UCB outperforms, and both UCB and Mean significantly outperform GCR. 

We tested a variant of UCB and Mean, where we incorporated the estimation of wind speed to compute the time to goal ($t_g$) as shown in Eqn.~\ref{eq:reward}. Due to the low quality estimation of wind speed at the areas far away from the aircraft's current position, using wind estimation to compute $t_g$ actually worsen the performance of both UCB and Mean. 

\begin{table}[t]
\begin{center}
\resizebox{0.45\textwidth}{!}{ 
    \begin{tabular}{llr}
   \toprule
    & South Carolina to Utah   & Seattle to Miami \\ 
    \midrule
    \textbf{UCB} & \textbf{21079.7}$\pm$1109.0 & \textbf{31333.1}$\pm$1269.0  \\ 
    \textbf{Mean} &  21183.3$\pm$1263.1 & {31716.5$\pm$1016.0} \\ 
    \textbf{GCR} & 33712.5$\pm$1852.1 &  48195.7$\pm$1952.7 \\ 
    \bottomrule
    \end{tabular}}
\end{center}
\vspace{-5pt}
\caption{Average traveling time (seconds) with standard deviation resulting from UCB, Mean and GCR under the real wind field setup.}
\label{tab:comparison_real}
\end{table}

\section{CONCLUSION}
We present UCB-Replanning, an online receding horizon based path planner that operates in an environment with latent information that can be modeled by Gaussian Processes. Equipped with a pre-computed trajectory library, at every iteration UCB-Replanning algorithm picks a trajectory to execute while collecting observations of the latent information on the fly to update the Gaussian Process.  UCB-Replanning leverages the idea of optimism in the face of uncertainty to tradeoff exploration and exploitation in a near-optimal manner and achieve no-regret property with respect to an optimal decision maker that has full access to the latent information of the environment.




\section*{APPENDIX}

\subsection{Proof of Lemma~\ref{lemma:high_prob_on_variable}}
\begin{proof}
The proof is essentially the same as Lemma 5.1 in \cite{srinivas2010gaussian}. For completeness we present the proof here. Fix $t$ and $\mathbf{x}\in D_t$. Note that under our assumption that $f$ is a sample from the prior of GP, we have $g(\mathbf{x})\sim \mathcal{N}(\mu_{t-1}(\mathbf{x}), \sigma_{t-1}(\mathbf{x})^2)$. Hence we have $(g(x) - \mu_{t-1}(\mathbf{x}))/\sigma_{t-1}(\mathbf{x})\sim \mathcal{N}(0,1)$. The proof of Lemma 5.1 in \cite{srinivas2010gaussian} shows that if $r\sim\mathcal{N}(0,1)$, we have $P(|r|\geq c) \leq \exp(-c^2/2), \forall c>0$.
This gives us the following result:
\begin{align}
Pr\big( |g(\mathbf{x}) - \mu_{t-1}(\mathbf{x})| \leq \beta_t^{1/2}\sigma_{t-1}(\mathbf{x}) \big)\geq 1 - \exp(-\beta_t/2). \nonumber
\end{align}We choose any sequence $\pi_t$ such that $\sum\frac{1}{\pi_t} = 1$. For instance we can set $\pi_t = \pi^2 t^2/6$. Now let us set $\exp(-\beta_t/2) = \frac{\delta}{\pi_t |D_t|}$, namely we set $\beta_t = 2\log(\frac{|D_t|\pi_t}{\delta})$. Now use union bound over all rounds from $1$ to $T$ and over all $LK$ waypoints in $D_t$, we can prove the above lemma. 
\end{proof}

\subsection{Proof of Theorem~\ref{them:main_them}}
\begin{proof}
Let us define event $B$ as 
\begin{align}
&\forall k\in [K],\forall t, |f_{t,k}(\{g(\mathbf{x}_{t,j}^k)\}_{j=0}^{L}) - f_{t,k}(\{\mu_{t-1}(\mathbf{x}_{t,j}^k)\}_{j=0}^{L})|\nonumber\\
&\leq l\beta_t^{1/2} \sum_{j=0}^{L}\sigma_{t-1}(\mathbf{x}_{t,j}^{k}), \nonumber
\end{align}and from Lemma.~\ref{lemma:ci_traj} we know that the probability of event $B$ happens is at least $1-\delta$. Below we show that event $B$ implies Theorem.~\ref{them:main_them}. For the rest of the proof, we assume we condition on that event $B$ happens. Consider round $t$. Note that $I_t$ is defined as:
\begin{align}
\label{eq:m_1}
& I_t = \arg\max_{k\in[K]}f_{t,k}(\{\mu_{t-1}(\mathbf{x}_{t,j}^k)\}_{j=0}^{L}) \nonumber\\
&\;\;\;\; + l\beta^{1/2}\sum_{j=0}^{L}\sigma_{t-1}(\mathbf{x}_{t,j}^{k}),  
\end{align}and $I_t^*$ is defined as:
\begin{align}
\label{eq:m_2}
I_t^* = \arg\max_{k\in[K]}f_{t,k}(\{g(\mathbf{x}_{t,j}^k)\}_{j=0}^{L}),
\end{align} namely the best trajectory one would pick at this round $t$ if $g$ is known. Now let us define the single step regret $r_t$ as:
\begin{align}
r_t = f_{t,I_t^*}\big(\{g(\mathbf{x}_{t,j}^{I_t^*})\}_{j=0}^{L}\big) - f_{t,I_t}\big(\{g(\mathbf{x}_{t,j}^{I_t})\}_{j=0}^{L}\big),
\end{align}namely the regret  one has by choosing $I_t$ instead of $I_t^*$ at round $t$. We can upper bound $r_t$ using Eqn.~\ref{eq:m_1} and \ref{eq:m_2} as follows:
\begin{align}
&r_t = f_{t,I_t^*}(\{g(\mathbf{x}_{t,j}^{I_t^*})\}_{j=0}^{L}) - f_{t,I_t}(\{g(\mathbf{x}_{t,j}^{I_t})\}_{j=0}^{L})\nonumber \\
&\leq f_{t,I_t^*}(\{\mu_{t-1}(\mathbf{x}_{t,j}^{I_t^*})\}_{j=0}^{L}) +l\beta_t^{1/2}\sum_{j=0}^{L}\sigma_{t-1}(\mathbf{x}_{t,j}^{I_t^*})\nonumber \\
&\;\;\;\;\;\;\;\; - (f_{t,I_t}(\{\mu_{t-1}(\mathbf{x}_{t,j}^{I_t})\}_{j=0}^{L}) -l\beta_t^{1/2}\sum_{j=0}^{L}\sigma_{t-1}(\mathbf{x}_{t,j}^{I_t})) \nonumber\\ 
&\;\;\;\;\;\;\;\;  (\texttt{event $B$ happens}) \nonumber\\
&\leq f_{t,I_t}(\{\mu_{t-1}(\mathbf{x}_{t,j}^{I_t})\}_{j=0}^{L}) +l\beta_t^{1/2}\sum_{j=0}^{L}\sigma_{t-1}(\mathbf{x}_{t,j}^{I_t})\nonumber \\
& \;\;\;\;\;\;\;\; - (f_{t,I_t}(\{\mu_{t-1}(\mathbf{x}_{t,j}^{I_t})\}_{j=0}^{L}) -l\beta_t^{1/2}\sum_{j=0}^{L}\sigma_{t-1}(\mathbf{x}_{t,j}^{I_t})) \;\;\;\; \nonumber
 \\ &\;\;\;\;\;\;\;\;\; (\texttt{Definition of $I_t$ from Eqn.~\ref{eq:m_1}}) \nonumber\\
& = 2l\beta_t^{1/2}\sum_{j=0}^{L-1}\sigma_{t-1}(\mathbf{x}_{t,j}^{I_t}).
\end{align}
The square of $r_t$ can be bounded as follows: 
\begin{align}
&r_t^2 = 4l^2\beta_t(\sum_{j=0}^{L}\sigma_{t-1}(\mathbf{x}_{t,j}^{I_t}))^2\leq 4l^2\beta_tL\sum_{j=0}^{L}\sigma_{t-1}(\mathbf{x}_{t,j}^{I_t})^2 \nonumber \\
&\leq 4l^2\beta_T L\sigma^{-2}\sum_{j=0}^{L}\sigma_{t-1}(\mathbf{x}_{t,j}^{I_t})^2\sigma^2 \nonumber \\
&= 4l^2\beta_T L\sigma^2\sum_{j=0}^{L}\Big[\frac{\sigma^{-2}\sigma_{t-1}(\mathbf{x}_{t,j}^{I_t})^2}{\log(1+\sigma^{-2}\sigma_{t-1}(\mathbf{x}_{t,j}^{I_t})^2)}\log\big(1\nonumber\\
&\;\;\;\;\;\; +\sigma^{-2}\sigma_{t-1}(\mathbf{x}_{t,j}^{I_t})^2\big)\Big] \nonumber \\
& \leq 4l^2\beta_TL\sigma^2 \frac{\sigma^{-2}}{\log(1+\sigma^{-2})}\sum_{j=0}^{L}\log\big(1+\sigma^{-2}\sigma_{t-1}(\mathbf{x}_{t,j}^{I_t})^2\big) \nonumber \\
&= 4l^2\beta_TL\sigma^2 C_1\sum_{j=0}^{L-1}\log\big(1+\sigma^{-2}\sigma_{t-1}(\mathbf{x}_{t,j}^{I_t})^2\big)
\end{align} where $C_1 = \sigma^{-2}/\log(1+\sigma^{-2})\geq 1$ and the third inequality comes from the fact that the function $x/\log(1+x)$ is non-decreasing when $\mathbf{x}>0$, and $\sigma^{-2}\sigma_{t-1}^2 \leq \sigma^{-2}$ because we assume $\sigma_{t-1}(\mathbf{x})^2\leq \kappa(\mathbf{x},\mathbf{x})\leq 1$ for any $\mathbf{x}$.  

Since the regret $\mathbf{R}_T = \sum_{t=1}^T r_t$, we must have $\mathbf{R}_T^2\leq T\sum_t r_t^2$. Using Lemma 5.3 and Lemma 5.4 from \cite{srinivas2010gaussian}, we can link $\mathbf{R}_T$ to the maximum information gain as follows:
\begin{align}
&\mathbf{R}_T^2 \leq 4l^2\beta_TL^2\sigma^2C_1 T\sum_{t=1}^T\sum_{j=0}^{L}\log(1+\sigma^{-2}\sigma_{t-1}(\mathbf{x}_{t,j}^{I_t})^2) \nonumber\\
&\;\;\;\;\;\;  \leq 4l^2\beta_TL^2\sigma^2C_1 T\gamma_T,   \nonumber
\end{align} where $\gamma_T$ is the maximum information gain defined as 
\begin{align}
&\gamma_T = \max_{A\subseteq \mathbb{X}, |A| = LT}I(v_A; g) \nonumber\\
&= \max_{A\subseteq \mathbb{X}, |A| = LT}H(v_A) - H(v_A | g),\nonumber
\end{align} where $H(x)$ is the entropy of the random variable $x$, $H(x|y)$ is the conditional entropy, $v_A = \{g(\mathbf{x}) + \epsilon\}_{x\in A}$ is the set of observations of $g(\mathbf{x})$ for all states $\mathbf{x}$ in set $A$. Namely $\gamma_T$ quantifies the maximum reduction in uncertainty about $g$
from revealing the observations of $g$ on $LT$ states.

Theorem 5 from \cite{srinivas2010gaussian} shows that $\gamma_T\leq O(d \log LT)$ when $\kappa(\mathbf{x},\mathbf{x}') = \mathbf{x}^T\mathbf{x}'$ and $\gamma_T\leq O((\log LT)^{d+1})$ when $\kappa(\mathbf{x},\mathbf{x}') = \exp(-c\|\mathbf{x}-\mathbf{x}'\|^2)$. Substitute these results to the above inequality, we prove the theorem. 
\end{proof}



\bibliographystyle{IEEEtran}
\bibliography{allbib}

\end{document}